\documentclass{article}

\PassOptionsToPackage{numbers, sort&compress}{natbib}

\usepackage[final]{nips_2018}

\usepackage[utf8]{inputenc} 
\usepackage[T1]{fontenc}    
\usepackage{hyperref}       
\usepackage{url}            
\usepackage{booktabs}       
\usepackage{amsfonts}       
\usepackage{nicefrac}       
\usepackage{microtype}      

\usepackage{breakurl}
\usepackage{caption}
\usepackage{wrapfig}
\usepackage{graphicx}
\usepackage{textcomp}
\usepackage{blindtext}

\usepackage{multirow}
\usepackage{dsfont}
\usepackage{multicol}
\usepackage{amsmath}
\usepackage{amsthm}
\usepackage{amssymb}
\usepackage{bm}
\usepackage{mathrsfs}
\usepackage{xcolor}
\usepackage{colortbl}

\usepackage{algorithm}
\usepackage{algpseudocode}
\usepackage{subcaption}

\usepackage{color}
\usepackage{soul}

\usepackage{pgfplots}
\pgfplotsset{
	compat=newest,
	plot coordinates/math parser=false,
	tick label style={font=\scriptsize, /pgf/number format/fixed},
	label style={font=\small},
	every axis/.append style={
		tick align=outside,
		clip mode=individual,
		scaled ticks=false,
		thick,
		tick style={semithick, black}
	}
}

\pgfkeys{/pgf/number format/.cd, set thousands separator={\,}}

\usetikzlibrary{calc}
\usepgfplotslibrary{external}
\newlength\figurewidth
\newlength\figureheight

\newlength\sbsfigurewidth
\newlength\sbsfigureheight

\usepackage{xspace}

\newcommand{\ens}{\acro{ENS}\xspace}
\newcommand{\bens}{batch-\ens}

\newcommand{\gb}{greedy-batch\xspace}

\newcommand{\bmg}{\acro{BMG}s\xspace}

\newcommand{\mR}{\mathcal{R}}
\newcommand{\mX}{\mathcal{X}}

\newcommand{\mc}[1]{\mathcal{#1}}
\newcommand{\data}{\mc{D}}

\newcommand{\given}{\mid}
\newcommand{\E}{\mathbb{E}}

\newcommand{\acro}[1]{\textsc{\MakeLowercase{#1}}}

\DeclareMathOperator*{\argmax}{arg\,max}

\newtheorem{theorem}{Theorem}

\newcommand{\opt}{\acro{OPT}\xspace}
\allowdisplaybreaks

\title{Efficient nonmyopic active search with applications in drug and materials discovery}

\author{
 Shali Jiang\\
  CSE, WUSTL\\
  St. Louis, MO 63130 \\
  \texttt{jiang.s@wustl.edu} \\
   \And
   Gustavo Malkomes \\
   CSE, WUSTL\\
   St. Louis, MO 63130 \\
   \texttt{luizgustavo@wustl.edu} \\
   \And
   Benjamin Moseley \\
   Tepper School of Business, CMU\\
   Pittsburgh, PA 15213 \\
   \texttt{moseleyb@andrew.cmu.edu}
   \And
  Roman Garnett \\
  CSE, WUSTL\\
  St. Louis, MO 63130 \\
  \texttt{garnett@wustl.edu} \\
}

\begin{document}

\maketitle

\begin{abstract}

Active search is a learning paradigm for actively identifying as many
members of a given class as possible. A critical target scenario is
high-throughput screening for scientific discovery, such as drug or
materials discovery. 
In this paper\footnote{This paper summarizes the contributions of two papers published at ICML 2017\cite{jiang2017efficient} and accepted at NIPS 2018\cite{jiang2018efficient}. Proofs, related work, and some experimental results are omitted.}, 
we approach this problem in Bayesian decision framework.
We first derive the Bayesian optimal policy under a natural utility, and establish a theoretical hardness of active search, proving that the optimal policy can not be approximated for any constant ratio.
We also study the batch setting for the first time, where a batch of $b>1$ points can be queried at each iteration. 
We give an asymptotic lower bound, linear in batch size, on the adaptivity gap: how much we could lose if we query $b$ points at a time for $t$ iterations, instead of one point at a time for $bt$ iterations.
We then introduce a novel approach to nonmyopic approximations of the optimal policy that admits efficient computation. 
Our proposed policy can automatically trade off exploration and exploitation, without relying on any tuning parameters. 
We also generalize our policy to batch setting, and propose two approaches to tackle the combinatorial search challenge. 
We evaluate our proposed policies on a large database of drug discovery and materials science. 
Results demonstrate the superior performance of our proposed policy in both sequential and batch setting; the nonmyopic behavior is also illustrated in various aspects. 

\end{abstract}


\section{Introduction}

In active search (\acro{AS}), we seek to sequentially inspect data to discover
as many members of a desired class as possible with a limited budget.  Formally,
suppose we are given a finite domain of $n$ elements $\mX = \{x_i\}_{i=1}^n$,
among which there is a rare, valuable subset $\mR \subset \mX$. We call the
members of this class \emph{targets} or \emph{positive items.} The identities of
the targets are unknown \emph{a priori,} but can be determined by querying an
expensive oracle that can compute $y =\mathds{1} \{x \in \mR\}$ for any
$x\in\mX$. Given a budget $T$ on the number of queries we can provide the
oracle, we wish to design a policy that sequentially queries items $\{x_t\} =
\{x_1, x_2, \dotsc, x_T\}$ to maximize the number of targets identified, $\sum
y_t$.
Many real-world problems can be naturally posed in terms of active search; drug
discovery \citep{garnett2015introducing, oglic_et_al_aaai_2017,
  oglic_et_al_mi_2018}, materials discovery \citep{jiang2017efficient}, and product recommendation
\citep{sutherland_et_al_kdd_2013} are a few examples.

Previous work \citep{garnett2012bayesian} has developed Bayesian optimal policies
for active search with a natural utility function.  Not surprisingly, this
policy is computationally intractable, requiring cost that grows exponentially
with the horizon.  To overcome this intractability, the authors of that work proposed using myopic lookahead policies in practice, which compute the optimal
policy only up to a limited number of steps into the future. This defines a family of policies ranging in complexity from completely greedy one-step lookahead to the optimal policy, which looks ahead to the depletion of the entire budget.
The authors demonstrated improved performance on active search over the greedy policy even when looking just two steps into the future, including in a drug-discovery setting \citep{garnett2015introducing}. The main limitation of these strategies is that they completely ignore what can happen beyond the chosen horizon, which for typical problems is necessarily limited to $\ell \leq 3$, even with aggressive pruning. More related work can be found in \citet{jiang2017efficient} and \citet{jiang2018efficient}.

In this paper, we first introduce the  Bayesian optimal policy for active search, and present a hardness result for this problem: no polynomial-time policy can achieve a constant factor approximation ratio with respect to the expected utility of the optimal policy. 

We also study batch active search for the first time, where a batch of $b>1$ points can be queried at a time. This is motivated by practical applications such as high throughput screening for drug discovery, where 96+ compounds can be processed at a time. 
Certainly this is more efficient, but the performance could be compromised for being less adaptive. 
So one interesting question is: how much do we lose?
We prove that: the optimal performance when we query one point at a time for $T$ iterations is at least $\Omega(b/\log T)$ times of that when we query $b$ points at a time for $T/b$ iterations.

We then introduce a novel \emph{nonmyopic} policy for active search that considers not only the immediate contribution of each unlabeled point
but also its potential impact on the remaining points that could be chosen afterwards. Our policy \emph{automatically} balances exploitation against exploration consistent with the labeling budget without requiring any parameters controlling this tradeoff. 

We also generalize our proposed policy to batch setting. The nonmyopia is automatically inherited, but the efficiency is not preserved due to combinatorial search. We propose two efficient approaches to approximately compute it. 

We compare our methods with several baselines by conducting experiments on numerous real datasets spanning drug discovery and material science, in both sequential and batch settings. We also illustrate the nonmyopic behavior of our proposed policies in various aspects. 
Our results thoroughly demonstrate that our policy typically significantly outperforms previously proposed active search approaches.


\section{Bayesian optimal policy for active search}
\label{sec:method}

We first introduce the \emph{optimal} policy for
active search under the general batch setting using the framework of Bayesian decision theory. Sequential active search would be a special case with batch size one \citep{garnett2012bayesian}. 
To cast batch active search into this framework, we express our preference over different datasets $\data = \bigl\{(x_i, y_i)\bigr\}$ through a natural utility: $u(\data) = \sum y_i$, which simply counts the number of targets in $\data$. Occasionally
we will use the notation $u(Y)$ for $u(\data)$ when $\data=(X,Y)$.  We now consider the problem
of sequentially choosing a set of $T$ (a given budget) points $\data$ with the
goal of maximizing $u(\data)$.  For each query we must
select a batch of $b$ ($b\ge1$) points and will then observe all their labels at the same
time.  We use $X_i = \{ x_{i, 1}, x_{i, 2}, \dotsc x_{i, b} \}$ to denote a
batch of points chosen during the $i$th iteration, and $Y_i = \{ y_{i, 1}, y_{i,
  2}, \dotsc y_{i, b}\}$ the corresponding labels.
We use $\data_i = \big\{(X_k, Y_k) \big\}_{k=1}^{i}$ to denote the observed data
after $i \leq t$ batch queries, where $t = \lceil T/b \rceil$.

We assume a probability model $\mc{P}$ is given, providing the posterior
marginal probability $\Pr(y\given x, \data)$ for any point $x\in\mc{X}$ and
observed dataset $\data$.  At iteration $i+1$ (given observations $\data_{i}$),
the Bayesian optimal policy chooses a batch $X_{i+1}$ maximizing the expected
utility at termination, recursively assuming optimal continued behavior:
\begin{equation} \label{eq:bayesian_optimal_policy}
	X_{i+1} = \argmax_{X} \E
	\bigl[
	u(\data_t \setminus \data_{i})
	\given
	X, \data_{i}
	\bigr].
\end{equation}
Note that the additive nature of our chosen utility allows us to ignore the
utility of the already gathered data in the expectation.

To derive the expected utility, we adopt the standard technique of backward
induction, as used by for example \citet{garnett2012bayesian} to analyze the
sequential case.  The base case is when only one batch is
left ($i=t-1$). The expected utility resulting from a proposed final batch $X$ is
then
\begin{equation}
\label{eq:greedy-batch}
\E
\bigl[
u(\data_t \setminus \data_{t-1})
\given
X, \data_{t - 1}
\bigr]
=
\E_{Y\given X, \data_{t-1}}
\bigl[
u(Y)
\bigr]
= {\textstyle
  \sum_{x\in X} \Pr(y=1\given x, \data_{t-1}) },
\end{equation}
where $\E_{Y\given X, \data_{i}}$ is the expectation over the joint posterior
distribution of $Y$ (the labels of $X$) conditioned on $\data_{i}$. In this
case, designing the optimal batch \eqref{eq:bayesian_optimal_policy} by
maximizing the expected utility is trivial: we select the points with the
highest probabilities of being targets, reflecting pure exploitation. This
optimal batch can then be found in $\mc{O}(n \log b)$ time using, e.g., min-heap of size $b$.

In general, when $i \le t-1$, the expected terminal utility resulting from
choosing a batch $X$ at iteration $i+1$ and acting optimally thereafter can be
written as a Bellman equation as follows:
\begin{equation}
\label{eq:expected_utility}
\E
\bigl[
	u(\data_t \setminus \data_{i})
	\given
	X, \data_{i}
\bigr]
=
\textstyle\sum_{x\in X} \Pr(y=1\given x, \data_{i})
	+
		\E_{Y\given X, \data_{i}}
		\Bigl[
		\max_{X'}
	\E\bigl[ u(\data_t \setminus \data_{i+1}) \given X', \data_{i+1} \bigr]
	\Bigr],
\end{equation}
where the first term represents the expected utility resulting immediately from
the points in $X$, and the second part is the expected future utility from the
following iterations.

The most interesting aspect of the Bayesian optimal policy is that these
immediate and future reward components in \eqref{eq:expected_utility} can be
interpreted as automatically balancing exploitation (immediate utility) and
exploration (expected future utility given the information revealed by the
present batch).

However, without further assumptions on the joint label distribution $\mc{P}$,
exact maximization of \eqref{eq:expected_utility} requires enumerating the whole
search tree of the form $ \data_{i} \rightarrow X_{i+1} \rightarrow Y_{i+1}
\rightarrow \cdots \rightarrow X_t \rightarrow Y_t. $ The branching factor of
the $X$ layers is $\binom{n}{b}$, as we must enumerate all possible batches. The
branching factor of the $Y$ layers is $2^b$, as we must enumerate all possible
labelings of a given batch.  So the total complexity of a na\"ive implementation
computing the optimal policy at iteration $i + 1$ would be a daunting
$\mc{O}\bigl( (2n)^{b(t-i)} \bigr)$.  The running time analysis in
\citep{garnett2012bayesian} is a special case of this result where $b=1$.

The optimal policy is clearly computationally infeasible, so we must resort to
suboptimal policies to proceed in practice. 
For sequential case ($b=1$), one typical workaround is to pretend there is only $\ell$ steps left and compute the policy in $\mc{O}^{\ell}$ \citep{garnett2012bayesian}. We will call these $\ell$-step lookahead policies. In our experiment we will consider $\ell=1,2$, and we refer to them as \emph{one-} or \emph{two-step} lookahead policy.
For batch case ($b>1$), we can compute a one-step policy by selecting the points with highest probabilities in $\mc{O}(n\log b)$ time, but looking even just one more step ahead would be infeasible due to combinatorial search. We will call batch one-step policy as \emph{greedy-batch}. Note $\ell$-step (for small $\ell$) policies are myopic since they can't see past the horizon of $\ell$. 

\subsection{Hardness of Approximation}
\citet{garnett2012bayesian} showed via an explicit construction that
the expected performance of the $\ell$-step policy can be arbitrarily
worse than any $m$-step policy with $\ell < m$, exploiting this
inability to ``see past'' the horizon.
We extend the above hardness result to show that no polynomial-time active search policy can be an approximation algorithm with respect to the optimal policy, in terms of
expected utility.
In particular, under the assumption that algorithms only have access to
a unit cost conditional marginal probability $\Pr(y=1\given x, \data)$ for any
$x$ and $\data$, where $|\data|$ is less than the
budget,\footnote{The optimal policy operates under these
	restrictions.}  then:
\begin{theorem}\label{thm:theorem}
	No polynomial-time policy for active search can have expected utility within a
	constant factor of the optimal policy.
\end{theorem}
\begin{proof}[Proof sketch]
The main idea is to construct a class of instances where a small ``secret'' set of elements encodes the locations of a large ``treasure'' of targets.
The probability of revealing the treasure is vanishingly small
without discovering the secret set; however, it is extremely unlikely to observe
any information about this secret set with polynomial-time effort.
See \citep{jiang2017efficient} for a detailed proof. 
\end{proof}

\subsection{Adaptivity gap}
For purely sequential policies (i.e., $b=1$), every point is chosen based on a
model informed by all previous observations.  However, for batch policies
$(b>1)$, points are typically chosen with less information available. For
example, in the extreme case when $b=T$, every point in our budget must be
chosen before we have observed anything, hence we might reasonably expect our
search performance to suffer. Clearly there must be an inherent cost to batch
policies compared to sequential policies due to a loss of adaptivity.  How much
is this cost?

We have proven the following lower bound on the inherent ``cost of parallelism''
in active search:
\begin{theorem}
	\label{thm:gap_bound}
	There exist active search instances with budget $T$,
	such that
	$\frac{\opt_1}{\opt_b}$ is $\Omega\bigl(\frac{b}{\log T} \bigr)$,
	where  $\opt_x$ is the expected number of targets found by the optimal batch policy with batch size $x \ge 1$.
\end{theorem}
\begin{proof}[Proof sketch]
	We construct a special type of active search instance where the location
        of a large trove of positives is encoded by a binary tree, and a search
        policy must take the correct path through the tree to decode a treasure
        map pointing to these points. We design the construction such that a
        sequential policy can easily identify the correct path by walking down
        the tree directed by the labels of queried nodes. A batch policy must
        waste a lot queries decoding the map as the correct direction is only
        revealed after constructing an entire batch. We show that even the
        optimal batch policy has a very low probability of identifying the location of the
        hidden targets quickly enough, so that the expected utility is much less than that of the optimal sequential policy.
        See \citep{jiang2018efficient} for a detailed proof.
\end{proof}

Thus the expected performance ratio between optimal sequential and batch
policies, also known as \emph{adaptivity gap} in the literature
\citep{asadpour2008stochastic}, is lower bounded linearly in batch size.  This
theorem is not only of theoretical interest: it can also provide practical
guidance on choosing batch sizes. Indeed, in drug discovery, modern
high-throughput screening technologies provide many choices for batch sizes;
understanding the inherent loss from choosing larger batch sizes provides
valuable information regarding the tradeoff between efficiency and cost.


\section{Efficient nonmyopic approximations}
We first illustrate our idea of efficient nonmyopic approximation in sequential case ($b=1$), and then generalize it to the batch ($b>1$) setting.
We have seen above how to myopically approximate the Bayesian optimal policy using
an $\ell$-step-lookahead approximate policy.
Such an approximation, however,
effectively assumes that the search procedure will terminate after the
next $\ell$ evaluations, which does not reward exploratory behavior
that improves performance beyond that horizon.
We propose to continue to exactly compute the
expected utility to some fixed horizon, but to approximate the
remainder of the search differently. We will approximate the expected
utility from any remaining portion of the search by assuming that any
remaining points, $\{x_{i+2}, x_{i+3}, \dotsc, x_{t}\}$, in our budget will be
selected \emph{simultaneously} in one big batch.
One rationale is if we assume that after observing $\data_{i+1}$,
the labels of all remaining unlabeled points are conditionally independent, then this approximation recovers the Bayesian optimal policy exactly.
This assumption might seem unrealistic
at first, but when many \emph{well-spaced} points are observed, we note they
might approximately ``D-separate'' the remaining unlabeled points.
Further, as we will demonstrate in Section \ref{sec:toy_problem}, \ens naturally encourages the selection of well-spaced points (targeted exploration) in the initial state of the search.

By exploiting linearity of expectation, it is easy to
work out the optimal policy for selecting such a simultaneous batch observation after iteration $i+1$: 
we simply select the points with the highest probability of being valuable.
The resulting approximation is
\begin{equation}
\label{approximate_remainder}
\max_{x'} \mathbb{E}\bigl[u(\data_t \backslash \data_{i+1} ) \given x', \data_{i+1}\bigr] \approx
{\textstyle \sum'_{t - i - 1}} \Pr(y = 1 \given x, \data_{i+1}),
\end{equation}
where the summation-with-prime symbol $\sum'_k$ indicates that we only
sum the largest $k$ values.

At iteration $i+1$, given $\data_i$, our proposed policy selects points by maximizing the approximate 
expected utility using:
\begin{multline}
\label{eq:approximate_search_utility}
\E
\bigl[
u(\data_t \backslash \data_{i}  )
\given
x_{i+1}, \data_{i}
\bigr]
\approx
\Pr(y_{i+1} = 1 \given x_{i+1}, \data_{i})
+ 
\underbrace{
	\E_{y_{i+1}}\Bigl[
	{\textstyle \sum'_{t-i-1}} \Pr\bigl(y = 1 \given x, \data_{i+1} \bigr)
	\Bigr]}_{\text{exploration, } {} < \, t-i-1}.
\end{multline}
We will call this policy \emph{efficient nonmyopic search} (\ens).
As in the optimal policy, we can interpret
\eqref{eq:approximate_search_utility} naturally as rewarding both
exploitation and exploration, where the exploration benefit is judged by
a point's capability to increase the top probabilities among
currently unlabeled points.  We note further that in
\eqref{eq:approximate_search_utility} the reward for exploration
\emph{naturally decreases over time} as the budget is depleted, exactly as in the optimal policy. In
particular, the very last point $x_t$ is chosen greedily by
maximizing probability, agreeing with the true optimal policy.
The second-to-last point is also guaranteed to match the optimal
policy.

Note that we may also use the approximation in
\eqref{approximate_remainder} as part of a finite-horizon lookahead
with $\ell > 1$, producing a family of increasingly expensive but
higher-fidelity approximations to the optimal policy, all retaining
the same budget consciousness. The approximation in
\eqref{eq:approximate_search_utility} is equivalent to a one-step
maximization of \eqref{approximate_remainder}. We will see in our
experiments that this is often enough to show massive gains in
performance, and that even this policy shows clear awareness of the
remaining budget throughout the search process, automatically and
dynamically trading off exploration and exploitation.
\subsection{Generalization to batch setting}
The generalization of \ens to batch setting is conceptually simple: how many targets would we expect to find if, after selecting the current batch, we spent the entire remaining budget simultaneously?  If this were the case, then similar to \eqref{eq:approximate_search_utility}, we have
(let $f(X \given \data_{i}) \equiv \E[ u(\data_t \setminus \data_{i}) \given X, \data_{i}]$):
\begin{equation} \label{eq:bens1}
	f(X \given \data_{i})
	=
	{\textstyle
	\sum_{x\in X} \Pr(y=1\given x, \data_{i}) }
	+
	{\textstyle
	\E_{Y\given X, \data_{i}}
	\left[
	 \sum'_{T-b-|\data_{i}|} \Pr\left( y'=1\given x', \data_{i}, X, Y \right)
	 \right].
	 }
\end{equation}

The nonmyopia of \eqref{eq:bens1} is automatically inherited in generalizing
from sequential to batch setting due to explicit budget awareness. Unfortunately, the efficiency of the sequential \ens policy is not preserved.  Direct maximization
of \eqref{eq:bens1} still requires combinatorial search over all subsets of size
$b$. Moreover, to evaluate a given batch, we need to enumerate all its possible
labelings ($2^b$ in total) to compute the expectation in the second
term. Accounting for the cost of conditioning and summing the top probabilities,
the total complexity would be $\mc{O}\bigl( (2n)^b\, n\log T\bigr)$.

We propose two strategies to tackle these computational problems below.

\textbf{Sequential simulation.}
The cost of computing the proposed batch policy has exponential dependence on the batch size $b>1$.
To avoid this, our first idea is to reduce the batch to sequential case ($b=1$).
We select points one at a time to add to a batch by maximizing the sequential \ens score (i.e., \eqref{eq:bens1} with $b=1$). We then use some fictional labeling oracle $\mc{L}\colon \mc{X}\to \{0,1\}$ to simulate its label and incorporate the observation into our dataset. We repeat this procedure until we have selected $b$ points. 
Note that we could use this basic construction replacing \ens by any other sequential policy $\pi$, such as the one-step or two-step Bayesian optimal policies \citep{garnett2012bayesian}.

We will see that the behavior of the fictional labeling oracle has large influence on the behavior of resulting search policies. Here we will consider four fictional oracles:
(1) sampling, where we randomly sample a label from its marginal distribution; 
(2) most-likely, where we assume the most-likely label; 
(3) pessimistic, where we always believe all labels are negative; and
(4) optimistic, where always believe all labels are positive.

Sequential simulation is a common \emph{heuristic} in similar settings like batch Bayesian optimization \citep{jiang2018efficient}.
Here we provide some mathematical rationale of this procedure in a special case, inspired by the work of \citet{wang2017thesis}: 
the batch constructed by sequentially simulating the greedy active search policy with a pessimistic oracle near-optimally maximizes the probability that \emph{at least one} of the points in the batch is positive. This is easy to prove using submodularity \citep{nemhauser1978analysis}. 
For more details, see \cite{jiang2018efficient}. 
Interestingly, the probability of at least one positive can be considered as an active search counterpart of a batch version of \emph{probability of improvement} for Bayesian optimization \citep{kushner_jbe_1964}.

\textbf{Greedy approximation.}
Our second strategy is motivated by our conjecture that \eqref{eq:bens1} is a monotone submodular function under reasonable assumptions.
If that is the case, then again a \emph{greedy} batch construction returns a batch with near-optimal score \citep{nemhauser1978analysis}.
We therefore propose to use a greedy algorithm to sequentially construct the batch by maximizing the marginal gain. That is, we begin with
an empty batch $X = \emptyset$. We then sequentially add $b$ points by adding the point maximizing the marginal gain:
\begin{equation}
	x
	=
	{\textstyle\argmax_x  \Delta_f(x \given X) },
	 \label{eq:batch-ens-greedy}
\end{equation}
where
\begin{equation}
\Delta_f(x \given X) = f(X \cup \{x\} \given \data_i) - f(X \given \data_i). \label{eq:marginal_gain}
\end{equation}
When $b$ is large, this procedure is still expensive to compute due to the expectation term in \eqref{eq:bens1}, requiring $\mc{O}(2^b)$ operations to compute exactly.
Here we approximate the expectation using Monte Carlo sampling with a small set of samples of the labels.
Specifically,
given a batch of points $X$, we approximate \eqref{eq:bens1} with samples
$S = \{\tilde{Y}: \tilde{Y} \sim  Y \given X, \data_i \}$:
\begin{equation}
f(X \given \data_i)
\approx
	\textstyle
	\sum_{x\in X} \Pr(y=1\given x, \data_i)
	+
	\textstyle
	\frac{1}{|S|}\sum_{Y\in S}
	\left[
	 \sum'_{T-b-|\data_i|} \Pr\left( y'=1\given x', \data_i, X, Y \right)
	 \right].
	  \label{eq:approx_bens1}
\end{equation}

We will call the batch policy described above \bens.
Note \bens using \emph{one} sample of the labels in a batch is similar to sequential simulation of \ens with the sampling oracle, though the two policies are motivated in different ways.


\subsection{Implementation and pruning}
All these policies can be implemented efficiently if we use a model with local structure such as $k$-nn (i.e., observing a point can only affect the probabilities of a very small subset of other points). Furthermore, we develop an aggressive pruning technique that resembles \emph{lazy evaluation}; we observe on drug discovery datasets, over 98\% of the candidate points can be pruned in each iteration on average. See \cite{jiang2018efficient} for more details.


\section{Experiments}
\label{sec:experiments}
\subsection{Nonmyopic Behavior} 
\label{sec:toy_problem}

\begin{figure}
	\centering
	\begin{subfigure}[b]{0.20\textwidth}
		\includegraphics{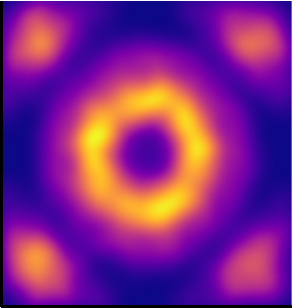}
		\subcaption{}
		\label{fig:toy_begin_ours}
	\end{subfigure}
	\hspace{0.5cm}
	\begin{subfigure}[b]{0.20\textwidth}
		\includegraphics{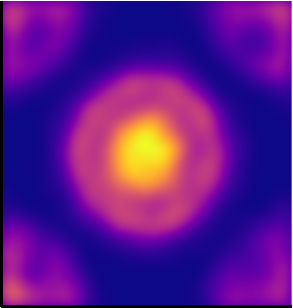}
		\subcaption{}
		\label{fig:toy_end_ours}
	\end{subfigure}
	\hspace{0.5cm}
	\begin{subfigure}[b]{0.20\textwidth}
		\includegraphics{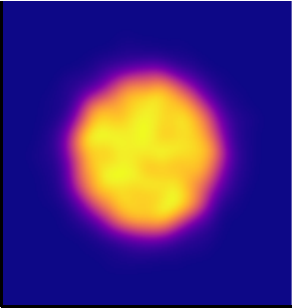}
		\subcaption{}
		\label{fig:toy_begin_myopic}
	\end{subfigure}
	\hspace{0.5cm}
	\begin{subfigure}[b]{0.20\textwidth}
		\includegraphics{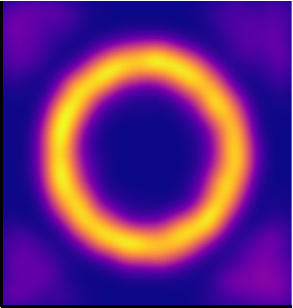}
		\subcaption{}
		\label{fig:toy_end_myopic}
	\end{subfigure}
	\caption{Kernel density estimates of the distribution of points chosen by \ens and two-step lookahead during two different time intervals. (a) \ens, first half. (b) \ens, second half. (c) two-step, first half. (d) two-step, second half.}
	\label{fig:toy_example}
\end{figure}
We first illustrate the nonmyopic behavior of \ens in contrast to the myopic two-step lookahead policy. 
We adapted the toy example presented by \citet{garnett2012bayesian}.  
Let $I\triangleq [0,1]^2$ be the unit square. We repeated the following
experiment 100 times. We selected 500 points iid uniformly at random
from $I$ to form the input space $\mX$.  We create an active search
problem by defining the set of targets $\mR \subseteq \mX$ to be all
points within Euclidean distance $\nicefrac{1}{4}$ from either the
center or any corner of $I$.
We took the closest point to the center
(always a target) as an initial training
set. We then applied \acro{ENS} and the two-step lookahead policies to sequentially select 200 further points for labeling.

Figure \ref{fig:toy_example} shows a kernel density
estimate of the distribution of locations selected by both methods during two time intervals. Figures \ref{fig:toy_example}(a--b)
correspond to our method; Figures \ref{fig:toy_example}(c--d) to two-step
lookahead.
Figures \ref{fig:toy_example}(a, c) consider the distribution of the first 100 selected locations;
Figures \ref{fig:toy_example}(b, d) consider the last 100.
The qualitative difference between these strategies is clear. The myopic policy focused on
collecting all targets around
the center (Figure \ref{fig:toy_example}(c)),
whereas our policy explores the boundaries of the center
clump with considerable intensity, as well as some
of the corners (Figure \ref{fig:toy_example}(a)).
As a result, our policy is capable of finding some of targets
in the corners, whereas two-step lookahead hardly ever can
(Figure \ref{fig:toy_example}(d)). We
can also see that the highest probability mass in
Figure \ref{fig:toy_example}(b) is the center, which
shows that our policy typically saves many high-probability
points until the end. On average,
the \acro{ENS} policy found about 40 more targets at termination
than the two-step lookahead policy.

\begin{wraptable}{R}{0.52\textwidth}
	\begin{minipage}{.48\textwidth}
	\caption{\bmg dataset: 
		Average number of targets found by the one- and two-step policies and \ens with different five budgets at specific time steps. The performance of the best method at each time waypoint is in bold.
	} \label{tab:bmg_vary_budget}
	\begin{tabular}{llllll} \toprule
		& \multicolumn{5}{c}{query number} \\ \cmidrule(l){2-6}
		policy	& 100	& 300	& 500	& 700	& 900	\\ \midrule
		one-step	& 90.8	& 273	& 450	& 633	& 798	\\
		two-step	& 91.0	& 273	& 452	& 632	& 802	\\
		\midrule
		\ens--900	& 89.0	& 270	& 453	& 635	& \textbf{815}	\\
		\ens--700	& 91.3	& 276	& 460	& \textbf{645}	\\
		\ens--500	& 92.4	& 279	& \textbf{466}	\\
		\ens--300	& 92.8	& \textbf{279}	\\
		\ens--100	& \textbf{94.5}	\\
		\bottomrule
	\end{tabular}

	\end{minipage}
\end{wraptable}

\subsection{Finding bulk metallic glasses} 
The goal here is to find novel alloys capable of forming bulk metallic glasses (\bmg).
Compared to crystalline alloys, \bmg have many desirable properties, including high toughness and good wear resistance.
This dataset consists of 118\,678 known alloys from the materials literature \cite{kawazoe_et_al_1997, ward_et_al_arxiv_2016}, among which 4\,746 (about 4\%) are known to exhibit glass-forming ability, which we define as positive/targets.

\textbf{Adaptation to budget.}
We conduct experiments on this dataset to demonstrate \ens' ability to adapt to the budget, compared to one- and two-step policies. 
We use $k$-nn model with $k=50$.
We select a single target uniformly at random to form an initial training set. 
We use each policy to sequentially select $t$ points for labeling.  The experiment was repeated 20 times, varying the initial seed target. 
We test for $t = 100, 300, 500, 700, 900$ and report the average number of targets found at these time points for each method in Table \ref{tab:bmg_vary_budget}.

We have the following observations from the table.
First, \ens~performs better than the myopic baseline policies for every budget.
Second, \ens~is able to adapt to the specified budget.
For example, when comparing performance after 100 queries, \ens-100 has located many more targets than the \ens~methods with greater budgets, which at that time are still strongly rewarding exploration.
A similar pattern holds when comparing other pairs of \ens~variations.

\subsection{Virtual drug screening} 
\begin{wrapfigure}{R}{0.52\textwidth}
	\begin{minipage}{.48\textwidth}
		\includegraphics{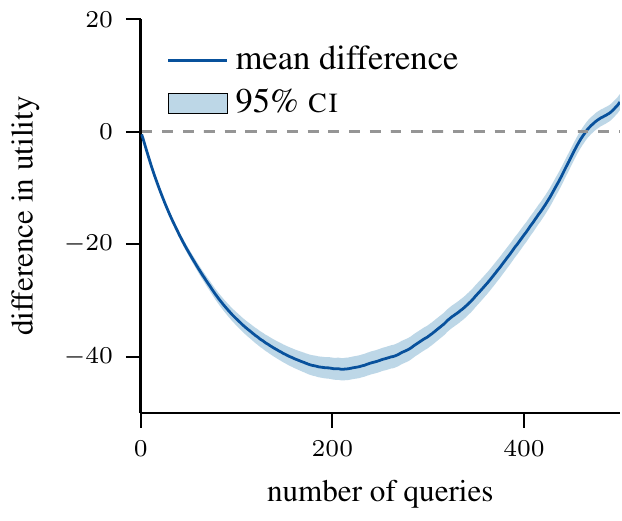}
		\caption{The average difference in cumulative targets found  between \ens and the two-step policy, averaged over 120 activity classes and 20 experiments on the \acro{ECFP4} fingerprint.}\label{fig:mean_diff}
	\end{minipage}
\end{wrapfigure}
\begin{table*}
	\centering
	\caption{Number of active compounds found by various active search policies at termination for each fingerprint, averaged over 120 active classes and 20 experiments.
		Also shown is the difference of performance between \ens~and two-step lookahead and the
		results of the corresponding paired $t$-test.}
	\label{tab:summary}
	\small
	\begin{tabular}{lcccccccc} \toprule
		&\multicolumn{4}{c}{policy}   &\multicolumn{3}{c}{$t$-test results}\\
		\cmidrule(r){2-5}  \cmidrule(l){6-9}
		fingerprint							&100-\acro{nn}   & one-step  & two-step & \ens & difference & $p$-value &  \multicolumn{2}{c}{95\% \acro{CI} } \\
		\midrule
		\acro{ECFP4}			& 189 & 289	& 297	& \textbf{303}	&    5.29	& $1.76 \times 10^{-3}$	& 2.01	& 8.56 \\
		\acro{G}pi\acro{DAPH3}	& 134 & 255	& 261	& \textbf{276}	&   14.8	& $3.90\times10^{-13}$	& 11.2	& 18.4 \\ \bottomrule
	\end{tabular}
\end{table*}

We conduct experiments on a massive database of chemoinformatic data.
The basic setting is to screen a large database of compounds searching for those that show binding activity against some biological target.  This is a basic component of drug-discovery pipelines. The dataset comprises
120 activity classes of human biological importance  selected from the Binding \acro{DB} \cite{liu2007bindingdb} database.
For each activity class, there are a small number of compounds with significant binding activity; the number of targets varies from 200 to 1\,488 across the activity classes.  From these we define 120 different active search problems.
There are also 100\,000 presumed inactive compounds selected at random from the \acro{ZINC} database \cite{zinc}; these are used as a shared negative class for each of these problems. For each compound, we consider two different feature representations,
also known as chemoinformatic fingerprints, called \acro{ECFP4} and \acro{G}pi\acro{DAPH3.} These fingerprints are binary vectors encoding the relevant chemical characteristics of the compounds; fingerprint similarities are computed via Jaccard index; see \citep{garnett2015introducing} for more details.
So in  total we have 240 active search problems, each with more than 100\,000 points, and with targets less than $1.5\%$.
We use $k$-nn model with $k=100$ for these datasets.

We perform comprehensive comparison on the $2\times 120$ virtual drug screening datasets. This time we fix the budget $t=500$. We again initialize the search by a random target, and repeat each experiment 20 times. 
We also report the performance of a baseline where we randomly sample a stratified sample of size 5\% of the database (${\sim}$5\,000 points, more than 10 times the budget of the active search policies).
From this sample, we train the same $k$-\acro{nn} model, compute the active probability of the remaining points, and query the 500 points with the highest posterior activity probabilities.

Table \ref{tab:summary} summarizes the results.
First we notice that all active search policies perform much better than the recall of a simple classification algorithm, even though they observe less than one-tenth the data.
The two-step policy is again better than the greedy policy for both fingerprints, which is consistent with the results reported in \cite{garnett2015introducing}.
The \ens policy performs significantly better than two-step lookahead; a two-sided paired $t$-test overwhelmingly rejects the hypothesis that the performance at termination is equal in both cases.

\textbf{Exploration vs. exploitation.}
Figure \ref{fig:mean_diff} shows the mean difference in cumulative targets found between \ens~and the two-step policy for the \acro{ECFP4} fingerprint. 
This plot again demonstrates the nonmyopic behavior of our proposed policy:
we very clearly observe the automatic trade-off between exploration and exploitation by our method.
In the initial stage of the search, we explore the space without
much initial reward, but around query 200,
our algorithm switches automatically to exploitation,
outperforming the myopic policy significantly at termination.

\subsection{Batch setting}
In this section, we evaluate sequential simulation of \ens and \bens against myopic baselines in batch active search. In total we evaluate 14 batch policies:
(1) \gb, coded as ``greedy'';
(2--13) sequential simulation, coded as ``ss\acro{-P-O}'', where \acro{P} (for policy) could be ``one'' (for one-step), ``two'' (for two-step), or ``\ens'',
	and \acro{O} (for oracle) could be ``s'' (sampling), ``m'' (most-likely), ``0'' (pessimistic, i.e., always-0), or ``1'' (optimistic, i.e., always-1);
(14) \bens.
Suggested by one of the the anonymous reviewers, we also compare these policies against another na\"ive baseline, which we call \emph{uncertain-greedy} batch (\acro{UGB}), where we build batches that simultaneously encourage exploration and exploitation by combining the most uncertain points and the highest probability points. We use a hyperparamter $r\in (0,1)$ to control the proportion, choosing the most uncertain points for $100r\%$ of the batch, and greedy points for the remaining $100(1-r)\%$ of the batch. We run this policy for $r\in \{0.1, 0.2, \dots, 0.9\}$, and show the best result among them.

\begin{table}
	\centering
	\caption{Results for 10 drug discovery datasets in batch setting: Average number of positive compounds found by the baseline \emph{uncertain-greedy} batch, \gb, sequential simulation and \bens policies. Each column corresponds to a batch size, and each row a policy. Each entry is an average over 200 experiments (10 datasets by 20 experiments). The budget $T$ is 500. Highlighted are the best (bold) for each batch size and those that are not significantly worse (blue italic) than the best  under one-sided paired $t$-tests with significance level $\alpha=0.05$.} \label{tab:drug_data}
	\vspace{1em}
	
\begin{tabular}{llllllllll}
\toprule
&{1} & {5} & {10} & {15} & {20} & {25} & {50} & {75} & {100}\\\hline
\acro{UGB} & - & 257.6 & 257.9 & 258.3 & 250.1 & 246.0 & 218.8 & 206.2 & 172.1 \\
{greedy} & 269.8   &  268.1   &  264.1   &  261.6   &  258.2   &  257.0   &  240.1   &  227.2   &  208.2   \\ \hline 
{ss-one-1} & 269.8   &  260.7   &  254.6   &  245.2   &  233.6   &  223.4   &  200.8   &  182.9   &  178.9   \\  
{ss-one-m} & 269.8   &  264.5   &  257.7   &  250.0   &  244.4   &  236.5   &  211.7   &  195.4   &  179.4   \\  
{ss-one-s} & 269.8   &  266.8   &  261.3   &  256.7   &  248.7   &  244.1   &  214.9   &  202.4   &  181.3   \\  
{ss-one-0} & 269.8   &  268.1   &  264.1   &  261.6   &  258.2   &  257.0   &  240.1   &  227.2   &  208.2   \\ \hline 
{ss-two-1} & 281.1   &  237.1   &  219.8   &  210.8   &  212.1   &  196.2   &  172.1   &  158.8   &  152.9   \\  
{ss-two-m} & 281.1   &  252.6   &  246.4   &  237.2   &  232.9   &  225.1   &  200.2   &  181.6   &  167.2   \\  
{ss-two-s} & 281.1   &  248.9   &  242.5   &  235.3   &  226.6   &  219.2   &  196.7   &  175.3   &  158.3   \\  
{ss-two-0} & 281.1   &  252.5   &  247.6   &  247.9   &  244.4   &  240.4   &  225.6   &  213.8   &  199.1   \\ \hline 
{ss-\ens-1} & \textbf{295.1  } &  269.4   &  247.9   &  227.2   &  223.1   &  210.3   &  185.3   &  152.6   &  148.7   \\  
{ss-\ens-m} & \textit{\textcolor{blue}{295.1  }} &  293.8   &  290.2   &  285.3   &  281.6   &  274.4   &  249.4   &  217.2   &  203.1   \\  
{ss-\ens-s} & \textit{\textcolor{blue}{295.1  }} &  289.9   &  278.3   &  269.8   &  262.6   &  255.0   &  220.8   &  185.5   &  161.2   \\  
{ss-\ens-0} & \textit{\textcolor{blue}{295.1  }} &  293.6   &  289.1   &  288.1   &  \textit{\textcolor{blue}{287.5  }} &  280.7   &  269.2   &  257.2   &  241.0   \\ \hline 
{batch-\ens-16} & \textit{\textcolor{blue}{295.1  }} &  \textbf{300.8  } &  \textbf{296.2  } &  293.9   &  \textbf{292.1  } &  \textit{\textcolor{blue}{288.0  }} &  275.8   &  \textit{\textcolor{blue}{272.3  }} &  252.9   \\  
{batch-\ens-32} & \textit{\textcolor{blue}{295.1  }} &  \textit{\textcolor{blue}{300.8  }} &  \textit{\textcolor{blue}{295.5  }} &  \textbf{297.9  } &  \textit{\textcolor{blue}{290.6  }} &  \textbf{288.8  } &  \textbf{281.4  } &  \textbf{275.5  } &  \textbf{263.5  } \\  
\bottomrule
\end{tabular}

\end{table}

We only conduct experiments on the first ten of the 120 \acro{ECFP}4 virtual drug screening datasets.
For each of the ten datasets, we still start with one random target, and repeat for 20 times. The budget is again $T=500$. 
We test  for batch sizes $b \in \{5, 10, 15, 20, 25, 50, 75, 100\}$, 
so the number of iterations $t = \lceil T/b \rceil \in \{100, 50, 34, 25, 20, 10, 7, 5\}$. 
We also show the results for sequential search ($b = 1$) as a reference. 
We test \bens with 16 and 32 samples, coded as \bens-16 and  \bens-32.
We show the number of positive compounds found in Table \ref{tab:drug_data}, averaged over the 10 datasets and 20 experiments each, so each entry in the table is an average over 200 experiments. We highlight the best result for each batch size in boldface. We conduct a paired $t$-test for each other policy against the best one, and also emphasize those that are not significantly worse than the best with significance level $\alpha=0.05$ in blue italics.

We highlight the following observations.
(1) The uncertain-greedy batch policy is worse than most of our proposed batch active search policies based on Bayesian decision framework, especially the nonmyopic ones. 
(2) The performance decreases as the batch size increases.
(3) Nonmyopic policies are consistently better than myopics ones;
in particular, \bens is a clear winner.
(4) For sequential simulation policies, the pessimistic oracle is almost always the best.

For \bens, we find batch-\ens with $32$ samples often performs better than with $16$, especially for larger batch sizes.
We have run \bens for $b=50$ with  $N\in\{2,4,8,16,32,64\}$, and find that the performance improves considerably as the number of samples increases, but the magnitude of this improvement tends to decrease with larger numbers. We believe 32 label samples offers a good tradeoff between efficiency and accuracy for $b=50$.
 
We now  discuss our observations in more detail. First we see all our proposed policies perform better than the heuristic uncertain-greedy batch, even if we optimistically assume the best hyperparameter of this policy (not to mention we hardly know what the best hyperparameter should be in practice).
Our framework based on Bayesian decision theory offers a more principled approach to batch active search (especially batch-\ens); and our methods are effectively hyperparameter-free (except the number of samples used in batch-\ens). In the following, we elaborate on the other three observations.

\begin{figure*}
	\centering
	\begin{subfigure}[b]{.45\textwidth}
		\centering
		\includegraphics{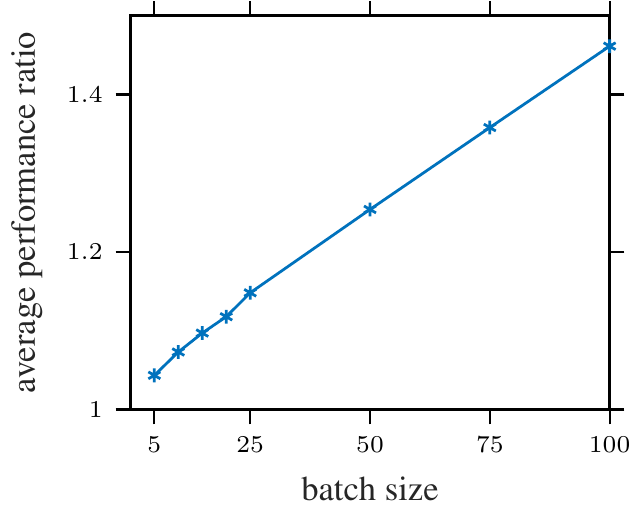}
		\subcaption{}
		\label{fig:empirical_performance_ratio}
	\end{subfigure}
	\hspace{0.5cm}
	\begin{subfigure}[b]{.45\textwidth}
		\centering
		\includegraphics{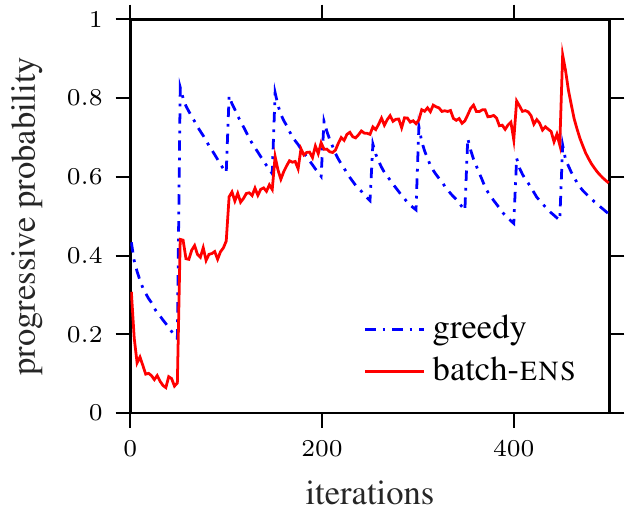}
		\subcaption{}
		\label{fig:probability_trace_drug_data_batch_size_50}
	\end{subfigure}
	\caption{(a) Average performance ratio between sequential policies and batch policies, as a function of batch size, produced using averaged results in Table \ref{tab:drug_data}.
		(b) Progressive probabilities of the chosen points of greedy and \bens-32, averaged over results for batch size 50 on all 10 drug discovery datasets and 20 experiments each.}
\end{figure*}

\textbf{Empirical adaptivity gap.}
Regardless of what policy is used, the performance in general degrades as the batch size increases. But how fast?
We average the results in Table \ref{tab:drug_data} over all policies for each batch size $b$ as an empirical surrogate for $\opt_b$ in Theorem \ref{thm:gap_bound}, and plot the resulting surrogate value of $\frac{\opt_1}{\opt_b}$ as a function of $b$ in Figure \ref{fig:empirical_performance_ratio}.
Although these policies are not optimal, the empirical performance gap matches our theoretical linear bound surprisingly well. Similar results for different budgets on a different dataset can be round in \citep{jiang2018efficient}.
These results could provide valuable guidance on choosing batch sizes.

Despite the overall trends in our results, we see some interesting exceptions. For example, in Table \ref{tab:drug_data}, batch-\ens with batch size 5 is significantly better than that with batch size 1, with a $p$-value of $0.02$ under a one-sided paired $t$-test. This is counterintuitive based on our analysis regarding the adaptivity gap. We conjecture that batch-\ens with larger batch sizes forces more (but not too much) exploration, potentially improving somewhat on sequential \ens in practice.

\textbf{Why is the pessimistic oracle better?}
Among the four fictional oracles, the pessimistic one usually performs the best for sequential simulation.
When combined with a greedy policy, we have provided some mathematical rationale before: sequential simulation then near-optimally maximizes the probability of unit improvement, which is a reasonable criterion.
Intuitively,
by always assuming the previously added points to be negative, the probabilities of nearby points are lowered, offering a repulsive force compelling later points to be located elsewhere, leading to a more diverse batch. This mechanism could help better explore the search space.  
Note this coincides with the idea of using repulsion for batch policy design for Bayesian optimization \citep{gonzalez2016batch}.

\textbf{Nonmyopic behavior revisited.}
To gain more insight into the nature of the myopic/nonmyopic behavior,
in Figure \ref{fig:probability_trace_drug_data_batch_size_50} we plot the probabilities   of the points chosen (at the iteration of being chosen) by the greedy and \bens-32 policies for batch size 50 across the drug discovery datasets. 
First, in each batch, the trend for greedy is not surprising, since every batch represents the top-50 points ordered by probabilities. For \bens, there is no such trend except in the last batch, where \bens naturally degenerates to greedy behavior.
Second, along the whole search process, greedy has a decreasing trend, likely due to over-exploitation in early stages.
On the other hand, \bens has an increasing trend.
This could be partly due to more and more positives being found.
More importantly, we believe this trend is in part a reflection of the nonmyopia of \bens:
in early stages, it tends to explore the search space, so low probability points might be chosen. As the remaining budget diminishes, it becomes more exploitive; in particular, the last batch is purely exploitive.



\section{Conclusion and future directions}
In this paper, we introduced a principled approach to active search, where the goal is to identify as many positive points as possible in a given labeling budget. Several theoretical results are established, such as the hardness of this problem and the adaptivity gap between sequential and batch optimal policies. We also developed an efficient nonmyopic policy that can automatically balance exploration and exploitation. Its superior performance and nonmyopic behavior is demonstrated on both simulated and real datasets.
We believe our theoretical and emprical analysis constitute a valuable step towards more-effective application of (batch) active search in various important domains such as drug discovery and materials science.

However, there are still many interesting open problems on this topic. 
First, our hardness result was proved without assuming any restrictions on the problem instances. So one natural question is: can we identify the conditions under which efficient algorithms with bounded approximation ratio exist?

Second, nonmyopic policies crucially rely on the model correctness. If the model is wrong, nonmyopia might hurt. Given that we often do not know the true model, how to design nonmyopic policies robust to model misspecification? 

Third, we studied active search in a setting where we need to maximize utility under cost constraints. What if we are required to minimize the cost under utility constraints? This is a common setting in active learning (e.g. \citep{Chen:2013}), but much less understood for active search, and very important in practice. 

\section*{Acknowledgments}
We would like to thank all the anonymous reviewers for valuable feedbacks.
\acro{SJ}, \acro{GM}, and \acro{RG} were supported by the National Science Foundation (\acro{NSF}) under award number \acro{IIA}--1355406.
\acro{GM} was also supported by the Brazilian Federal Agency for Support and Evaluation of Graduate Education (\acro{CAPES}).
\acro{BM} was supported by a Google Research Award and by \acro{NSF} under awards \acro{CCF}--1830711, \acro{CCF}--1824303, and \acro{CCF}--1733873.

\newpage
\bibliography{ms}
\bibliographystyle{abbrvnat}
\end{document}